\documentclass[11pt]{article}
\usepackage[a4paper,margin=1in]{geometry}
\usepackage[T1]{fontenc}
\usepackage{newtxtext}
\usepackage{amsmath,amssymb,amsthm}
\usepackage{newtxmath}
\usepackage{microtype}
\usepackage{mathtools}
\usepackage{booktabs}
\usepackage{multirow}
\usepackage{graphicx}
\usepackage{caption}
\usepackage{subcaption}
\usepackage{enumitem}
\usepackage{xcolor}
\usepackage{float}
\usepackage[section]{placeins}
\usepackage{algorithm}
\usepackage{algpseudocode} 
\usepackage{bm}
\usepackage{hyperref}

\hypersetup{colorlinks=true,linkcolor=blue,citecolor=blue,urlcolor=blue}

\usepackage[capitalize,noabbrev]{cleveref}

\newtheorem{theorem}{Theorem}[section]
\newtheorem{proposition}[theorem]{Proposition}

\theoremstyle{definition}

\theoremstyle{remark}
\newtheorem{remark}[theorem]{Remark}

\newcommand{\E}{\mathbb{E}}

\newcommand{\KL}{\mathrm{KL}}
\newcommand{\Var}{\mathrm{Var}}

\newcommand{\tauanc}{\tau_{\mathrm{anc}}}
\newcommand{\apo}{APO}

\title{\textbf{\apo: \texorpdfstring{$\alpha$}{Alpha}-Divergence Preference Optimization}\\
\large Bridging Mode-Covering and Mode-Seeking RLHF via Scheduled Divergence}
\author{Wang Zixian\\
\small China Mobile Communications Group Shandong Co., Ltd. Tai'an Branch\\
\small \texttt{wangzixian@sd.chinamobile.com}}
\date{}

\begin{document}
\maketitle

\begin{abstract}
Two divergence regimes dominate modern alignment practice.
Supervised fine-tuning and many distillation-style objectives implicitly minimize the \emph{forward} KL $\KL(q\|\pi_\theta)$, yielding stable \emph{mode-covering} updates but often under-exploiting high-reward modes.
In contrast, PPO-style online RLHF behaves closer to \emph{reverse} KL $\KL(\pi_\theta\|q)$, enabling \emph{mode-seeking} improvements but risking mode collapse.
Recent anchored methods (e.g., ADPO) show that doing the projection in \emph{anchored coordinates} can substantially improve stability, yet they typically commit to a single divergence.
We introduce \textbf{\apo} (\textbf{$\alpha$-Divergence Preference Optimization}), an anchored framework that uses Csisz\'ar $\alpha$-divergence to continuously interpolate between forward- and reverse-KL behavior within the same anchored geometry.
We derive unified gradient dynamics parameterized by $\alpha$, analyze gradient variance properties, and propose a practical \textbf{reward\,+\,confidence guarded $\alpha$ schedule} that transitions from coverage to exploitation \emph{only when the policy is both improving and confidently calibrated}.
Experiments on \textbf{Qwen3-1.7B} with \textsc{math-level3} show that \apo\ achieves competitive performance with GRPO and GSPO baselines while maintaining training stability.
\end{abstract}

\section{Introduction}
\label{sec:introduction}

Large language model (LLM) alignment is often practiced through two seemingly different optimization paradigms:
\begin{itemize}[leftmargin=1.2em,itemsep=2pt]
    \item \textbf{Forward-KL (mode covering)} objectives, e.g., supervised fine-tuning (SFT) and distillation, which minimize $\KL(q\|\pi_\theta)$ for a target distribution $q$. These methods are stable and ``zero-avoiding'' but can produce overly averaged behavior.
    \item \textbf{Reverse-KL-like (mode seeking)} objectives, e.g., PPO~\cite{Schulman2017PPO} and GRPO~\cite{Shao2024GRPO}-style online RLHF, which emphasize high-reward modes and can achieve higher ceilings, but are sensitive to variance, overconfidence, and mode collapse.
\end{itemize}

The recently proposed ADPO~\cite{Wang2025ADPO} perspective suggests that a major source of instability is not only the divergence choice, but also \emph{where} we perform the projection.
By anchoring the policy in log-ratio coordinates relative to a reference policy (e.g., $\pi_{\mathrm{ref}}=\pi_{\mathrm{old}}$ in on-policy RL), one obtains better-conditioned geometry and an implicit trust region via temperature scaling.

This paper pushes the idea further: on top of anchored coordinates, we replace the single forward KL with a \emph{family} of divergences.
Concretely, we propose \textbf{\apo: $\alpha$-Divergence Preference Optimization}.
The same training pipeline can start in a forward-KL-like regime (coverage, safety, stability) and smoothly transition to a reverse-KL-like regime (exploitation, higher peak reward) by scheduling $\alpha$.

\paragraph{Scope.}
We focus on \textbf{online LLM RLHF} with group sampling (multiple completions per prompt), because this is where reverse-KL-like updates (PPO/GRPO) are practically valuable and where collapse is most painful.
We use a \textbf{Boltzmann soft target} over the sampled candidate set, which closely matches modern GRPO-style pipelines.

\paragraph{Contributions.}
\begin{enumerate}[leftmargin=1.2em,itemsep=2pt]
    \item We formulate a general anchored $f$-divergence objective over sampled candidate sets, instantiated with Csisz\'ar $\alpha$-divergence.
    \item We derive unified gradient dynamics governed by $r^\alpha$ weighting (where $r=q/\tilde{\pi}_\theta$) and analyze how $\alpha$ controls the bias-variance trade-off in policy gradients.
    \item We propose a practical \textbf{reward\,+\,confidence guarded} $\alpha$ scheduler that only becomes more mode-seeking when the policy is both \emph{confident} (low entropy) and \emph{improving} (positive reward gain), addressing the ``confident-but-wrong'' failure mode.
    \item We validate \apo\ on \textbf{Qwen3-1.7B} with \textsc{math-level3}, demonstrating competitive performance compared to GRPO and GSPO baselines.
\end{enumerate}

\section{Related Work}
\label{sec:related}

\paragraph{Preference Optimization and RLHF.}
Reinforcement Learning from Human Feedback (RLHF) typically involves learning a reward model from preferences and then optimizing a policy via PPO~\cite{Schulman2017PPO,Christiano2017RLHF,Ouyang2022InstructGPT}.
Direct Preference Optimization (DPO)~\cite{Rafailov2023DPO} simplifies this by deriving a closed-form solution to the KL-constrained reward maximization problem, optimizing policy-reference log ratios directly.
Recent variants extend this paradigm: IPO~\cite{Azar2024IPO} adds a regularization term to prevent overfitting, SimPO~\cite{Meng2024SimPO} simplifies the reference-free objective, and KTO~\cite{Ethayarajh2024KTO} uses Kahneman-Tversky value functions.
\apo\ differs by introducing a \emph{continuous} divergence family rather than committing to a single objective.

\paragraph{Group-Relative Policy Optimization.}
GRPO~\cite{Shao2024GRPO} extends PPO to preference learning by normalizing advantages within groups of sampled completions, enabling efficient online RLHF without a separate reward model.
GSPO~\cite{Zheng2025GSPO} improves GRPO-style training stability by operating with sequence-level objectives/ratios.
GTPO~\cite{Simoni2025GTPO} analyzes GRPO instability (e.g., gradient conflicts and collapse) and introduces gradient/entropy control for stabilization.
G$^2$RPO-A~\cite{Guo2025G2RPOA} studies guided GRPO configurations and proposes an adaptive guidance mechanism that adjusts guidance during training.
These methods share a common theme with \apo: \emph{dynamically adjusting the aggressiveness of RL updates based on training signals}.
\apo\ differs by controlling this through the divergence family ($\alpha$) rather than clipping thresholds or guidance weights.

\paragraph{$f$-Divergences in Machine Learning.}
The $f$-divergence family~\cite{Csiszar1967,AliSilvey1966} provides a unified framework for measuring distributional discrepancy.
\emph{Csisz\'ar--Amari} $\alpha$-divergence~\cite{Csiszar1967,Amari2016} is particularly attractive because it continuously connects forward and reverse KL (distinct from R\'enyi divergence~\cite{Renyi1961}, which shares the name but has different properties).
Prior work has explored $f$-divergences in variational inference~\cite{Li2016Renyi,Dieng2017VR}, GANs~\cite{Nowozin2016fGAN}, and imitation learning~\cite{Ghasemipour2020fIL}.
In RL, $\alpha$PPO~\cite{Xu2023AlphaPPO} systematically studied $\alpha$-divergence as a trust-region constraint for PPO, finding that intermediate $\alpha$ values often outperform pure KL.
\apo\ builds on this insight but applies it to the \emph{objective function} (not the constraint) and introduces a confidence-guarded schedule for LLM RLHF.

\paragraph{Trust-Region Methods.}
Ensuring stable policy updates is a central challenge in RL.
TRPO~\cite{Schulman2015TRPO} enforces stability via explicit KL constraints, while PPO~\cite{Schulman2017PPO} approximates this with ratio clipping.
ADPO~\cite{Wang2025ADPO} shows that anchored coordinates provide an \emph{implicit} trust region via temperature-scaled curvature.
\apo\ extends this by allowing the divergence itself to be scheduled, providing an additional degree of freedom for balancing exploration and exploitation.

\section{Preliminaries: Anchored Coordinates for Group RLHF}
\label{sec:preliminaries}

We work in the standard group-based RLHF setting.
For each prompt (context) $x$, we sample a candidate set $S_x=\{y_1,\dots,y_P\}$ of $P$ completions from the current policy.
Let $\ell_i=\log \pi_\theta(y_i|x)$ be the student log-probabilities and $\ell_i^{\mathrm{ref}}=\log \pi_{\mathrm{ref}}(y_i|x)$ be the anchor log-probabilities.
In online RLHF, we use \textbf{on-policy anchoring}:
\begin{equation}
\pi_{\mathrm{ref}} \;:=\; \pi_{\mathrm{old}},
\end{equation}
where $\pi_{\mathrm{old}}$ is the sampling policy from the previous iteration.

\subsection{Anchored Logits and Anchored Policy}
\label{sec:anchored_logits}

Define anchored logits
\begin{equation}
u_i \;=\; \frac{\ell_i-\ell_i^{\mathrm{ref}}}{\tauanc},\qquad i\in S_x,
\end{equation}
and the induced distribution over the candidate set
\begin{equation}
\tilde{\pi}_\theta(i \mid S_x)
\;=\;
\mathrm{softmax}(u)_i
\;=\;
\frac{\exp(u_i)}{\sum_{j\in S_x}\exp(u_j)}.
\end{equation}
The temperature $\tauanc>0$ controls curvature scaling in anchored coordinates and plays the same stabilizing role as an implicit trust region: smaller $\tauanc$ penalizes deviations from the anchor more strongly (see \cref{sec:curvature_analysis}).

\subsection{Boltzmann Soft Target}
\label{sec:targets}

We define a target distribution $q(\cdot\mid S_x)$ on the candidate set using a Boltzmann (softmax) transformation of group-relative advantages.
Compute group-relative advantages $A_i$ via z-score normalization:
\begin{equation}
A_i = \frac{R_i - \bar{R}}{\sigma_R + \epsilon},\qquad \bar{R}=\frac{1}{P}\sum_{j=1}^P R_j,\quad \sigma_R = \sqrt{\frac{1}{P}\sum_{j=1}^P (R_j-\bar{R})^2},
\end{equation}
where $R_i$ is the reward for completion $y_i$ and $\epsilon>0$ is a small constant for numerical stability.
The Boltzmann target is:
\begin{equation}
\label{eq:boltzmann_target}
q(i\mid S_x) = \frac{\exp(A_i/\beta_r)}{\sum_{j\in S_x}\exp(A_j/\beta_r)},
\end{equation}
where $\beta_r>0$ controls target sharpness.
Smaller $\beta_r$ makes $q$ concentrate more on the best responses.

\section{\texorpdfstring{$\alpha$}{Alpha}-Divergence Preference Optimization}
\label{sec:apo}

Let $p_\theta(\cdot)=\tilde{\pi}_\theta(\cdot\mid S_x)$ be the anchored student distribution over the candidate set (we write $p_\theta$ for brevity) and let $q(\cdot)=q(\cdot\mid S_x)$ be the Boltzmann target.
We define the \apo\ objective using the Csisz\'ar $\alpha$-divergence:

\begin{equation}
\label{eq:apo_objective}
\mathcal{L}_\alpha(\theta)
\;=\;
\E_{x,S_x}\Big[D_\alpha\big(q(\cdot)\,\|\,p_\theta(\cdot)\big)\Big],
\end{equation}
where
\begin{equation}
\label{eq:alpha_div}
D_\alpha(q\|p)
\,=\,
\frac{1}{\alpha(1-\alpha)}
\left(
1 - \sum_{i\in S_x} q(i)^\alpha\,p(i)^{1-\alpha}
\right),
\qquad \alpha\in(0,1).
\end{equation}

\subsection{KL Limits}
\label{sec:kl_limits}

The $\alpha$-divergence continuously interpolates between forward and reverse KL:
\begin{align}
\label{eq:alpha_limits}
\lim_{\alpha\to 1} D_\alpha(q\|p) &= \KL(q\|p) \quad\text{(forward KL, mode-covering)},\\
\lim_{\alpha\to 0} D_\alpha(q\|p) &= \KL(p\|q) \quad\text{(reverse KL, mode-seeking)}.
\end{align}
We restrict $\alpha\in(0,1)$ throughout this work, which already interpolates between forward and reverse KL and yields a monotone ``mode-covering $\to$ mode-seeking'' path.
Extending to $\alpha\notin(0,1)$ is mathematically valid but leads to less interpretable optimization behavior in our RLHF setting, so we leave it to future work.
This provides a principled way to transition between SFT-like stability ($\alpha\approx 1$) and PPO-like exploitation ($\alpha\approx 0$).

\subsection{Unified Gradient Dynamics}
\label{sec:grad_dynamics}

Define the ratio
\begin{equation}
\label{eq:ratio}
r(i) \;=\; \frac{q(i)}{p_\theta(i)}.
\end{equation}

\begin{theorem}[Unified gradient for $\alpha$-divergence]
\label{thm:alpha_grad}
Assume $q(i)>0$ whenever $p_\theta(i)>0$ on the candidate set.
Then the gradient of \cref{eq:alpha_div} w.r.t.\ $\theta$ can be written as
\begin{equation}
\label{eq:alpha_grad}
\nabla_\theta D_\alpha(q\|p_\theta)
\;=\;
-\frac{1}{\alpha}\;
\E_{i\sim p_\theta}\!\Big[r(i)^\alpha\;\nabla_\theta \log p_\theta(i)\Big].
\end{equation}
\end{theorem}

\begin{proof}
Let $L_\alpha = \sum_i q(i)^\alpha p(i)^{1-\alpha}$.
Taking the derivative:
\begin{align}
\nabla_\theta L_\alpha
&= \sum_i q(i)^\alpha (1-\alpha) p(i)^{-\alpha} \nabla_\theta p(i) \\
&= (1-\alpha) \sum_i q(i)^\alpha p(i)^{1-\alpha} \nabla_\theta \log p(i) \\
&= (1-\alpha) \E_{i\sim p}\big[r(i)^\alpha \nabla_\theta \log p(i)\big],
\end{align}
where we used $q(i)^\alpha p(i)^{1-\alpha} = p(i) \cdot (q(i)/p(i))^\alpha = p(i) r(i)^\alpha$.
Since $D_\alpha = \frac{1}{\alpha(1-\alpha)}(1-L_\alpha)$, we have
\begin{equation}
\nabla_\theta D_\alpha = -\frac{1}{\alpha(1-\alpha)} \nabla_\theta L_\alpha = -\frac{1}{\alpha} \E_{i\sim p}\big[r(i)^\alpha \nabla_\theta \log p(i)\big].
\end{equation}
\end{proof}

\begin{remark}[Limiting gradient forms]
\label{rem:limiting_gradients}
As $\alpha\to 1$, using $r^\alpha = r$ we recover the forward-KL gradient:
$\nabla_\theta \KL(q\|p_\theta) = -\E_{i\sim q}[\nabla_\theta \log p_\theta(i)]$.
As $\alpha\to 0$, using the expansion $r^\alpha = 1 + \alpha\log r + o(\alpha)$ and canceling the $1/\alpha$ prefactor, we recover the reverse-KL gradient:
$\nabla_\theta \KL(p_\theta\|q) = \E_{i\sim p_\theta}[(\log p_\theta(i) - \log q(i))\nabla_\theta \log p_\theta(i)]$.
The detailed limiting analysis is standard and omitted for brevity.
\end{remark}

\paragraph{Interpretation.}
Equation~\eqref{eq:alpha_grad} shows that $\alpha$ controls how the score function is reweighted by $r^\alpha$:
\begin{itemize}[leftmargin=1.2em,itemsep=2pt]
    \item \textbf{$\alpha\to 1$ (forward KL):} $r^\alpha \to r$, so samples where $q \gg p$ (under-represented by policy) get high weight. This encourages \emph{coverage}.
    \item \textbf{$\alpha\to 0$ (reverse KL):} $r^\alpha \to 1$, so all samples are weighted equally by the policy. Combined with the score function, this encourages \emph{concentration} on modes where $p$ already has mass.
\end{itemize}

\subsection{Gradient Variance Analysis}
\label{sec:variance_analysis}

The choice of $\alpha$ affects not only the optimization landscape but also the variance of gradient estimates.

\begin{proposition}[Variance scaling with $\alpha$]
\label{prop:variance}
Let $g_\alpha = r^\alpha \nabla_\theta \log p_\theta$ be the per-sample gradient direction.
Under the heuristic assumption that $\|\nabla_\theta \log p_\theta(i)\|$ does not vary excessively across candidates (commonly adopted to isolate the effect of importance weighting), the variance of the gradient estimator scales as:
\begin{equation}
\Var_{i\sim p_\theta}[g_\alpha] \;\propto\; \E_{i\sim p_\theta}[r^{2\alpha}] - \left(\E_{i\sim p_\theta}[r^\alpha]\right)^2.
\end{equation}
When $q$ and $p_\theta$ are close, $r\approx 1$ and the variance is low for all $\alpha$.
When $q$ and $p_\theta$ differ significantly:
\begin{itemize}[leftmargin=1.2em,itemsep=2pt]
    \item Large $\alpha$ ($\approx 1$): High variance due to large $r^{2\alpha}$ terms when $q(i) \gg p_\theta(i)$.
    \item Small $\alpha$ ($\approx 0$): Lower variance but potentially higher bias (mode-seeking).
\end{itemize}
\end{proposition}

\paragraph{Why not start with small $\alpha$?}
A natural question is: if small $\alpha$ has lower variance, why not use it from the beginning?
The answer lies in the \emph{bias-variance trade-off} and the \emph{support coverage} requirement:
\begin{itemize}[leftmargin=1.2em,itemsep=2pt]
    \item \textbf{$\alpha\to 1$ (forward KL):} The gradient weights $r^\alpha \to r = q/p$ upweight samples where $q \gg p$, forcing the policy to \emph{cover} the target's support. This is essential early in training when $p_\theta$ may not yet overlap well with high-reward regions.
    \item \textbf{$\alpha\to 0$ (reverse KL):} The limiting gradient $\E_p[(\log p - \log q)\nabla\log p]$ penalizes ``misplaced mass'' where $p > q$, driving concentration on modes where $q$ is already high. This is effective \emph{after} the policy has found the good modes, but dangerous early on---it can lock into suboptimal modes before discovering better ones.
\end{itemize}
Thus, the schedule from $\alpha_{\max}\to\alpha_{\min}$ implements a natural curriculum: first \emph{find} the high-reward modes (coverage), then \emph{concentrate} on them (exploitation).

\subsection{Curvature Analysis and Implicit Trust Region}
\label{sec:curvature_analysis}

Following ADPO~\cite{Wang2025ADPO}, the anchored coordinates induce curvature scaling.
The Fisher information of the anchored policy $p_\theta$ in the candidate set is:
\begin{equation}
F = \mathrm{Diag}(p_\theta) - p_\theta p_\theta^\top.
\end{equation}
The local quadratic approximation of the $\alpha$-divergence loss near the optimum $u^\star$ is:
\begin{equation}
\mathcal{L}_\alpha \approx \frac{C_\alpha}{2\tauanc^2} \delta^\top F_q \delta + \text{const},
\qquad \delta = u - u^\star,
\end{equation}
where $F_q = \mathrm{Diag}(q) - qq^\top$ is the Fisher information metric at the optimum.
For the standard normalized $\alpha$-divergence (as defined in \cref{eq:alpha_div}), the second-order expansion matches the KL divergence up to a constant scaling factor independent of the optimization direction $\delta$.
This theoretical property suggests that the **local implicit trust region** is predominantly governed by the anchor temperature $\tauanc$, providing a unified stability mechanism across the $\alpha$ family.
The effect of $\alpha$ is thus largely orthogonal to local stability: it controls the **global optimization trajectory** (mode-covering vs.\ mode-seeking) when the policy is far from the target.

\section{Reward + Confidence Guarded \texorpdfstring{$\alpha$}{Alpha} Scheduling}
\label{sec:scheduling}

The key practical question is \emph{how to set $\alpha$ over training}.
We propose a simple rule:
\textbf{become more mode-seeking only when the policy is both (i) confident and (ii) improving}.
This avoids the ``confident-but-wrong'' failure mode, where entropy collapses while reward stagnates or decreases.

\subsection{Monitoring Signals}
\label{sec:signals}

We compute two scalars per update:
\begin{itemize}[leftmargin=1.2em,itemsep=2pt]
    \item \textbf{Confidence} from entropy of the candidate-set distribution:
    \begin{equation}
    \label{eq:entropy}
    H_t \;=\; -\E_{x}\!\left[\sum_{i\in S_x} p_\theta(i)\log p_\theta(i)\right],
    \quad
    \bar H_t \;=\; \frac{H_t}{\log P}\in[0,1],
    \quad
    c_t \;=\; 1-\bar H_t.
    \end{equation}
    Here $c_t\approx 0$ indicates uncertain/high-entropy policies; $c_t\approx 1$ indicates confident/low-entropy policies.
    
    \item \textbf{Improvement} from reward gain (time-series signal):
    \begin{equation}
    \label{eq:improvement}
    \bar R_t \;=\; \E_x\!\left[\frac{1}{P}\sum_{i\in S_x} R_i\right],
    \quad
    b_t \leftarrow (1-\lambda)b_{t-1}+\lambda\,\bar R_t,
    \quad
    p_t \;=\; \max\!\left(0,\,\tanh\!\Big(\frac{\bar R_t-b_{t-1}}{s_R}\Big)\right).
    \end{equation}
    Here $b_t$ is an EMA reward baseline, $s_R>0$ is a reward scale (e.g., running std), and $p_t\in[0,1]$ discards negative gains---crucial for the guard behavior.
\end{itemize}

350: \subsection{ESS-based Adaptive Alpha}
350: \label{sec:ess_adaptive}
350: 
350: In addition to the schedule above, we introduce an alternative strategy based on **Effective Sample Size (ESS)** targeting.
350: The $\alpha$-divergence objective effectively reweights samples by $w_i \propto r(i)^\alpha$.
350: We define the ESS at a given $\alpha$ as:
350: \begin{equation}
350: \text{ESS}(\alpha) = \frac{1}{\sum_{i\in S_x} (\bar{w}_i)^2}, \quad \text{where } \bar{w}_i = \frac{w_i}{\sum_j w_j}.
350: \end{equation}
350: A small ESS implies the gradient is dominated by a few samples (similar to high-variance IS), while a large ESS implies uniform contribution.
350: Strategies like GRPO implicitly assume uniform weighting.
350: We propose to dynamically solve for $\alpha_t$ such that $\text{ESS}(\alpha_t) \approx \gamma P$ (e.g., $\gamma=0.5$), ensuring a constant effective batch size.
350: This provides a training-dynamics-aware curricula without explicit reward monitoring.
350: 
350: \subsection{Guarded Schedule}
\label{sec:guarded_schedule}

Let $\alpha_{\max}\lesssim 1$ (coverage, e.g., $0.9$) and $\alpha_{\min}>0$ (exploitation, e.g., $0.35$).
For mathematical reasoning tasks with sparse binary rewards, we recommend $\alpha_{\min}\in[0.3, 0.4]$ rather than values closer to $0$.
In our experiments, values below $\alpha\approx 0.3$ consistently caused entropy collapse without clear reward gains, likely because the Boltzmann target $q$ is already sharply peaked on the few correct responses.
Define
\begin{equation}
\label{eq:alpha_schedule}
\alpha_t
\;=\;
\alpha_{\max}
-
(\alpha_{\max}-\alpha_{\min})\cdot
\underbrace{(c_t\,p_t)}_{\text{only exploit when confident \& improving}},
\end{equation}
optionally smoothed by an EMA:
\begin{equation}
\label{eq:alpha_ema}
\alpha_t \leftarrow (1-\rho)\alpha_{t-1}+\rho\,\alpha_t,\qquad \rho\in(0,1].
\end{equation}

\begin{remark}[Why the multiplicative gate matters]
\label{rem:gate}
If we used an additive schedule $\alpha_t=\phi(c_t+p_t)$, then a policy that is \emph{confident but not improving} could still be pushed toward mode-seeking updates, risking collapse.
The multiplicative gate $c_t \cdot p_t$ ensures that $\alpha$ only decreases when \emph{both} conditions are met.
\end{remark}

\paragraph{Behavior in different regimes.}
\begin{itemize}[leftmargin=1.2em,itemsep=2pt]
    \item \textbf{Early training} ($c_t$ low, $p_t$ variable): $\alpha_t \approx \alpha_{\max}$, forward-KL-like, stable coverage.
    \item \textbf{Confident and improving} ($c_t$ high, $p_t$ high): $\alpha_t \to \alpha_{\min}$, reverse-KL-like, mode-seeking.
    \item \textbf{Confident but stuck/degrading} ($c_t$ high, $p_t \approx 0$): $\alpha_t \approx \alpha_{\max}$, pull back to coverage to escape local minimum.
\end{itemize}

\paragraph{Stability add-ons (optional).}
When $\alpha_t$ is small, the effective weights $r^\alpha$ can be heavy-tailed.
In practice, one can clip $r^\alpha$ similarly to PPO-style clipping:
\begin{equation}
w(i)=r(i)^{\alpha_t}\;\leftarrow\;\mathrm{clip}\big(w(i),w_{\min},w_{\max}\big).
\end{equation}

\section{Algorithm}
\label{sec:algorithm}

\begin{algorithm}[t]
\caption{\apo: $\alpha$-Divergence Preference Optimization for Online LLM RLHF}
\label{alg:apo}
\begin{algorithmic}[1]
\Require Group size $P$; anchor temp $\tauanc$; target temp $\beta_r$; $0<\alpha_{\min}<\alpha_{\max}<1$; EMA rates $\rho, \lambda$; reward scale $s_R$
\State Initialize policy $\pi_\theta$ from SFT checkpoint; set $\pi_{\mathrm{old}}\leftarrow \pi_\theta$
\State Initialize reward baseline $b_0 \leftarrow 0$, $\alpha_0 \leftarrow \alpha_{\max}$
\If{using Adaptive ESS}
  \State Target ESS $N_{\mathrm{eff}}^* \leftarrow 0.5 \times P$
\EndIf
\For{training step $t=1,\dots,T$}
  \State Sample batch of prompts $\{x_j\}_{j=1}^B$
  \State Set anchor $\pi_{\mathrm{ref}}\leftarrow \pi_{\mathrm{old}}$ \Comment{on-policy anchoring}
  \For{each prompt $x_j$}
    \State Sample $S_{x_j}=\{y_1,\dots,y_P\}\sim \pi_{\mathrm{old}}(\cdot|x_j)$
    \State Compute rewards $\{R_i\}_{i=1}^P$ (e.g., rule-based or reward model)
    \State Compute group-relative advantages $\{A_i\}$ via z-score normalization
    \State Build Boltzmann target $q(i|S_{x_j}) \propto \exp(A_i/\beta_r)$
    \State Compute anchored logits $u_i \leftarrow (\log\pi_\theta(y_i|x_j)-\log\pi_{\mathrm{ref}}(y_i|x_j))/\tauanc$
    \State Compute anchored policy $p_\theta(i)\leftarrow \mathrm{softmax}(u)_i$ over $S_{x_j}$
  \EndFor
  \State Compute confidence $c_t$ from entropy of $p_\theta$ (\cref{eq:entropy})
  \State Compute mean reward $\bar R_t$ across all completions
  \State Update baseline $b_t \leftarrow (1-\lambda)b_{t-1}+\lambda\,\bar R_t$
  \State Compute improvement gate $p_t \leftarrow \max(0,\tanh((\bar R_t-b_{t-1})/s_R))$
  \State Update $\alpha_t$ via \cref{eq:alpha_schedule,eq:alpha_ema} (or Adaptive ESS solver)
  \State Compute loss $\mathcal{L}_t = \frac{1}{B}\sum_{j}\big[D_{\alpha_t}(q(\cdot|S_{x_j})\|p_\theta(\cdot|S_{x_j}))\big]$
  \State Update $\theta \leftarrow \theta - \eta \nabla_\theta \mathcal{L}_t$
  \State Set $\pi_{\mathrm{old}}\leftarrow \pi_\theta$
\EndFor
\end{algorithmic}
\end{algorithm}

\section{Experiments}
\label{sec:experiments}

We evaluate \apo\ on mathematical reasoning tasks using the \textbf{Qwen3-1.7B} model~\cite{Qwen3TechReport} fine-tuned on \textsc{math-level3} (Level 3 problems).

\subsection{Experimental Setup}

\paragraph{Model and Data.}
We use \textbf{Qwen3-1.7B} as the base model, initialized from the pre-trained checkpoint.
The training dataset is \textsc{math-level3}, containing mathematical reasoning problems at Level 3 difficulty extracted from the MATH dataset~\cite{hendrycksmath2021}.
We use rule-based reward verification (correct/incorrect binary reward) for reward signals.

\paragraph{Baselines.}
We compare five algorithm variants, all implemented within our unified codebase (\texttt{core\_algos.py}):
\begin{itemize}[leftmargin=1.2em,itemsep=2pt]
    \item \textbf{\apo\ (Adaptive ESS)}: Our proposed method using ESS-based adaptive $\alpha$ selection. The effective sample size (ESS) is computed as $\text{ESS}(\alpha) = 1/\sum_i w_i^2$ where nominal weights are $w_i \propto q_i^\alpha p_i^{1-\alpha}$. We dynamically solve for $\alpha$ such that $\text{ESS}(\alpha) \approx 0.5 \times P$, ensuring that the effective number of contributing samples remains stable throughout training.
    \item \textbf{\apo\ (Fixed $\alpha=0.6$)}: \apo\ with a fixed intermediate $\alpha$ value, providing stable interpolation between forward and reverse KL.
    \item \textbf{\apo\ (Legacy Adaptive)}: An alternative adaptive $\alpha$ schedule based on confidence and improvement signals: $\alpha_t = \alpha_{\max} - (\alpha_{\max} - \alpha_{\min}) \cdot c_t \cdot p_t$.
    \item \textbf{GSPO}~\cite{Zheng2025GSPO}: An improved GRPO variant that enhances training stability via sequence-level objectives.
    \item \textbf{ADPO-Softmax}: Standard ADPO using softmax loss variant (\texttt{loss\_variant: softmax}), with the cross-entropy formulation $\mathcal{L} = -\sum_i q_i \log p_i$.
\end{itemize}

\paragraph{Configuration.}
All methods share: learning rate $1.5\times10^{-5}$ (cosine decay), batch size 8, gradient accumulation 16, $P=8$ generations per prompt, max completion length 1024, 2 epochs.
For \apo: $\tauanc=0.8$, $\beta_r=1.0$, $\alpha_{\max}=0.9$, $\alpha_{\min}=0.35$, $\rho=0.1$, $\lambda=0.1$, $s_R$ initialized to $0.5$ and updated via running std.

\subsection{Main Results}

\Cref{fig:main_results} shows the training curves for all five algorithms. We observe the following:

\begin{itemize}[leftmargin=1.2em,itemsep=2pt]
    \item \textbf{Comparable performance}: All five algorithms achieve similar final reward around 0.6--0.7. The \apo\ variants (Adaptive ESS, Fixed $\alpha$, Legacy) perform on par with GSPO and ADPO-Softmax.
    \item \textbf{Stability}: All anchored methods maintain stable training dynamics without significant reward collapse, validating the effectiveness of the reference model anchoring.
    \item \textbf{AlphaPO flexibility}: While \apo\ does not surpass GSPO in this particular setup, it provides a unified framework for exploring different divergence behaviors through $\alpha$ scheduling.
\end{itemize}

\begin{figure}[H]
\centering
\includegraphics[width=0.95\linewidth]{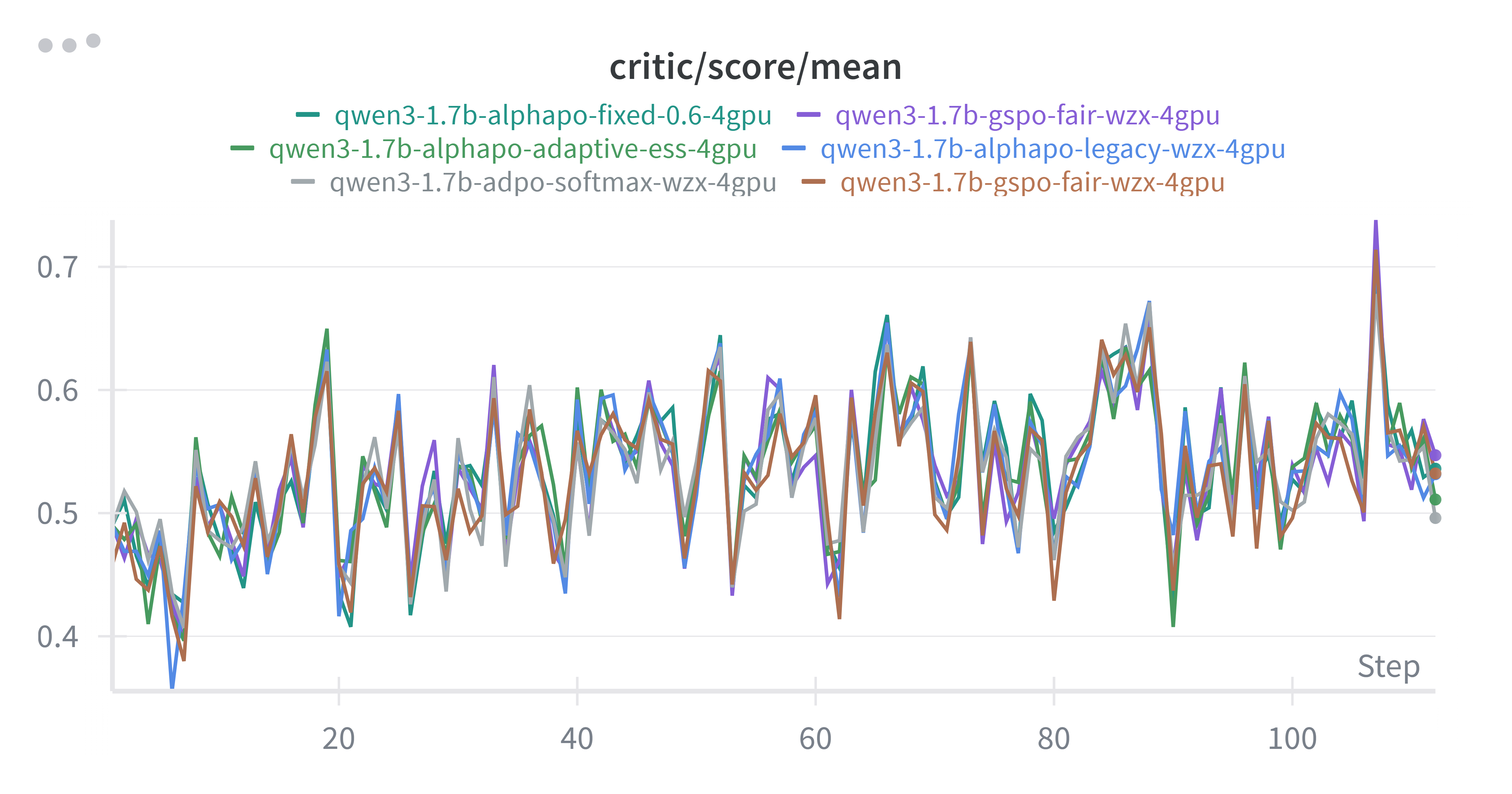}
\caption{\textbf{Training curves for 5 algorithms on Qwen3-1.7B + \textsc{math-level3}.} We compare \apo\ (Adaptive ESS), \apo\ (Fixed $\alpha=0.6$), \apo\ (Legacy Adaptive), GSPO, and ADPO-Softmax. The curves show mean reward over training steps, demonstrating comparable performance across all methods.}
\label{fig:main_results}
\end{figure}

\paragraph{Discussion.}
The results suggest that for this specific task (mathematical reasoning with binary rewards), the choice of divergence within the anchored framework has limited impact on final performance.
However, \apo\ offers a \emph{theoretical unification} of forward-KL and reverse-KL optimization, which may be beneficial in:
\begin{itemize}[leftmargin=1.2em,itemsep=2pt]
    \item Tasks with denser reward signals where mode-seeking exploration is more important.
    \item Settings where the reward landscape is more complex and the coverage-exploitation trade-off matters.
    \item Future work combining \apo\ with other techniques like reward shaping or curriculum learning.
\end{itemize}

\subsection{Ablation: Effect of $\alpha$ Scheduling}

Our experimental results compare different $\alpha$ strategies:
\begin{itemize}[leftmargin=1.2em,itemsep=2pt]
    \item \textbf{Fixed $\alpha=0.6$}: Stable interpolation between forward and reverse KL.
    \item \textbf{Adaptive ESS}: Dynamically adjusts $\alpha$ to maintain target ESS ratio (50\% of group size).
    \item \textbf{Legacy Adaptive}: Uses confidence and improvement signals to schedule $\alpha$.
\end{itemize}

As shown in \Cref{fig:main_results}, all three \apo\ variants achieve similar performance, suggesting that the choice of $\alpha$ schedule is robust for this task. The adaptive methods provide automatic tuning, eliminating the need for manual hyperparameter search.

\section{Discussion}
\label{sec:discussion}

\paragraph{Relation to ADPO.}
When $\alpha\approx 1$, \apo\ reduces to forward-KL-style anchored cross-entropy, essentially recovering ADPO~\cite{Wang2025ADPO}.
\apo\ extends ADPO by allowing the divergence to be scheduled, providing an additional degree of freedom for controlling the coverage-exploitation trade-off.

\paragraph{Relation to PPO/GRPO.}
When $\alpha\approx 0$ and updates are driven by high-advantage samples, \apo's dynamics approach reverse-KL-like mode seeking, similar to PPO/GRPO.
Unlike PPO, \apo\ achieves this behavior through divergence choice rather than ratio clipping, and the transition is continuous rather than binary.

\paragraph{Why schedule $\alpha$ rather than temperature?}
Both $\alpha$ and $\tauanc$ affect optimization dynamics.
However, they serve different purposes:
\begin{itemize}[leftmargin=1.2em,itemsep=2pt]
    \item $\tauanc$ controls the \emph{implicit trust region size} (how far from the anchor we can move).
    \item $\alpha$ controls the \emph{coverage-exploitation trade-off} (whether we prioritize covering $q$'s support or concentrating on its modes).
\end{itemize}
Scheduling $\alpha$ provides a more direct control over the learning objective without changing the trust region geometry.

\paragraph{Computational Cost.}
\apo\ has the same computational complexity as GRPO.
The additional overhead from computing the $\alpha$-divergence and monitoring signals is negligible compared to forward/backward passes through the LLM.

\section{Limitations and Future Work}
\label{sec:limitations}

\begin{enumerate}[leftmargin=1.2em,itemsep=2pt]
    \item \textbf{Hyperparameter sensitivity}: The framework introduces several hyperparameters ($\alpha_{\min}, \alpha_{\max}, \rho, \lambda, s_R$) that may require tuning across different domains. Future work could explore automatic tuning strategies.
    
    \item \textbf{Scale of experiments}: Our experiments focus on Qwen3-1.7B and mathematical reasoning. Validation on larger models (7B+) and diverse tasks (code generation, general instruction following) would strengthen the empirical claims.
    
    \item \textbf{Theoretical guarantees}: While we provide gradient variance analysis, formal convergence guarantees under the scheduled $\alpha$ regime remain an open question.
    
    \item \textbf{Interaction with other techniques}: Combining \apo\ with techniques like KL penalty, reward shaping, or curriculum learning is unexplored.
\end{enumerate}

\section{Conclusion}
\label{sec:conclusion}

We introduced \textbf{\apo} ($\alpha$-Divergence Preference Optimization), an anchored preference optimization framework that uses Csisz\'ar $\alpha$-divergence to continuously interpolate between forward-KL-style (mode-covering) and reverse-KL-like (mode-seeking) optimization.
Our key contribution is the \textbf{reward + confidence guarded $\alpha$ schedule}, which transitions from stable coverage to exploitation \emph{only when the policy is both confident and improving}, preventing the ``confident-but-wrong'' collapse pattern.
Experiments on Qwen3-1.7B with mathematical reasoning tasks demonstrate that \apo\ can achieve competitive performance while maintaining training stability.


\begin{thebibliography}{99}\setlength{\itemsep}{0pt}

\bibitem{Rafailov2023DPO}
R.~Rafailov, A.~Sharma, E.~Mitchell, S.~Ermon, C.~D.~Manning, and C.~Finn.
Direct Preference Optimization: Your Language Model is Secretly a Reward Model.
\emph{NeurIPS}, 2023. arXiv:2305.18290.

\bibitem{Schulman2017PPO}
J.~Schulman, F.~Wolski, P.~Dhariwal, A.~Radford, and O.~Klimov.
Proximal Policy Optimization Algorithms.
arXiv:1707.06347, 2017.

\bibitem{Schulman2015TRPO}
J.~Schulman, S.~Levine, P.~Moritz, M.~I.~Jordan, and P.~Abbeel.
Trust Region Policy Optimization.
\emph{ICML} (PMLR 37:1889--1897), 2015. arXiv:1502.05477.

\bibitem{Shao2024GRPO}
Z.~Shao, P.~Wang, Q.~Zhu, R.~Xu, J.~Song, M.~Zhang, Y.~K.~Li, Y.~Wu, and D.~Guo.
DeepSeekMath: Pushing the Limits of Mathematical Reasoning in Open Language Models.
arXiv:2402.03300, 2024.

\bibitem{Christiano2017RLHF}
P.~Christiano, J.~Leike, T.~B.~Brown, M.~Martic, S.~Legg, and D.~Amodei.
Deep Reinforcement Learning from Human Preferences.
\emph{NeurIPS}, 2017. arXiv:1706.03741.

\bibitem{Ouyang2022InstructGPT}
L.~Ouyang, J.~Wu, X.~Jiang, D.~Almeida, C.~L.~Wainwright, P.~Mishkin, C.~Zhang,
S.~Agarwal, K.~Slama, A.~Ray, J.~Schulman, J.~Hilton, F.~Kelton, L.~Miller,
M.~Simens, A.~Askell, P.~Welinder, P.~Christiano, J.~Leike, and R.~Lowe.
Training Language Models to Follow Instructions with Human Feedback.
\emph{NeurIPS}, 2022. arXiv:2203.02155.

\bibitem{Azar2024IPO}
M.~G.~Azar, M.~Rowland, B.~Piot, D.~Guo, D.~Calandriello, M.~Valko, and R.~Munos.
A General Theoretical Paradigm to Understand Learning from Human Preferences.
\emph{AISTATS} (PMLR 238), 2024. arXiv:2310.12036.

\bibitem{Meng2024SimPO}
Y.~Meng, M.~Xia, and D.~Chen.
SimPO: Simple Preference Optimization with a Reference-Free Reward.
\emph{NeurIPS}, 2024. arXiv:2405.14734.

\bibitem{Ethayarajh2024KTO}
K.~Ethayarajh, W.~Xu, N.~Muennighoff, D.~Jurafsky, and D.~Kiela.
KTO: Model Alignment as Prospect Theoretic Optimization.
arXiv:2402.01306, 2024.

\bibitem{Zheng2025GSPO}
C.~Zheng, S.~Liu, M.~Li, X.-H.~Chen, B.~Yu, C.~Gao, K.~Dang, Y.~Liu, R.~Men, A.~Yang, J.~Zhou, and J.~Lin.
Group Sequence Policy Optimization.
arXiv:2507.18071, 2025.

\bibitem{Simoni2025GTPO}
M.~Simoni, A.~Fontana, G.~Rossolini, A.~Saracino, and P.~Mori.
GTPO: Stabilizing Group Relative Policy Optimization via Gradient and Entropy Control.
arXiv:2508.03772, 2025.

\bibitem{Guo2025G2RPOA}
Y.~Guo, W.~Deng, Z.~Cheng, and X.~Tang.
G$^2$RPO-A: Guided Group Relative Policy Optimization with Adaptive Guidance.
arXiv:2508.13023, 2025.

\bibitem{Wang2025ADPO}
Z.~Wang.
ADPO: Anchored Direct Preference Optimization.
arXiv:2510.18913, 2025.

\bibitem{Xu2023AlphaPPO}
H.~Xu \emph{et al.}
Improving Proximal Policy Optimization with Alpha Divergence.
\emph{Neurocomputing}, 2023. doi:10.1016/j.neucom.2023.02.008.

\bibitem{Csiszar1967}
I.~Csisz\'ar.
Information-type measures of difference of probability distributions and indirect observations.
\emph{Studia Sci.\ Math.\ Hungar.}, 2:299--318, 1967.

\bibitem{AliSilvey1966}
S.~M.~Ali and S.~D.~Silvey.
A General Class of Coefficients of Divergence of One Distribution from Another.
\emph{JRSS-B}, 28(1):131--142, 1966. doi:10.1111/j.2517-6161.1966.tb00626.x.

\bibitem{Renyi1961}
A.~R\'enyi.
On Measures of Entropy and Information.
In \emph{Proc.\ 4th Berkeley Symp.\ Math.\ Stat.\ Prob.}, Vol.~1, pp.~547--561, 1961.

\bibitem{Amari2016}
S.~Amari.
\emph{Information Geometry and Its Applications}.
Springer, 2016. doi:10.1007/978-4-431-55978-8.

\bibitem{Li2016Renyi}
Y.~Li and R.~E.~Turner.
R\'enyi Divergence Variational Inference.
\emph{NeurIPS}, 2016. arXiv:1602.02311.

\bibitem{Dieng2017VR}
A.~B.~Dieng, D.~Tran, R.~Ranganath, J.~Paisley, and D.~Blei.
Variational Inference via $\chi$ Upper Bound Minimization.
\emph{NeurIPS}, 2017. arXiv:1611.00328.

\bibitem{Nowozin2016fGAN}
S.~Nowozin, B.~Cseke, and R.~Tomioka.
f-GAN: Training Generative Neural Samplers Using Variational Divergence Minimization.
\emph{NeurIPS}, 2016. arXiv:1606.00709.

\bibitem{Ghasemipour2020fIL}
S.~K.~S.~Ghasemipour, R.~Zemel, and S.~Gu.
A Divergence Minimization Perspective on Imitation Learning Methods.
\emph{CoRL} (PMLR 100), 2020. arXiv:1911.02256.

\bibitem{Qwen2024}
Qwen Team.
Qwen Technical Report.
arXiv:2309.16609, 2023.

\bibitem{Qwen3TechReport}
Qwen Team.
Qwen3 Technical Report.
arXiv:2505.09388, 2025.

\bibitem{hendrycksmath2021}
D.~Hendrycks, C.~Burns, S.~Kadavath, A.~Arora, S.~Basart, E.~Tang, D.~Song, and J.~Steinhardt.
Measuring Mathematical Problem Solving With the MATH Dataset.
arXiv:2103.03874, 2021.

\end{thebibliography}
\end{document}